\documentclass[11pt]{article}
\usepackage{amsfonts}
\usepackage{mathrsfs}
\usepackage{amsmath,amssymb}
\usepackage{mathtools}

\usepackage[latin1]{inputenc}
\usepackage{amsmath,amssymb}

\usepackage{amsthm}
\usepackage{latexsym}
\usepackage{epstopdf}
\usepackage{geometry}  
\usepackage{pict2e,picture}
\usepackage{bigints}  
\usepackage[thinlines]{easytable}
\usepackage{enumerate}
\usepackage{array}
\usepackage{fix-cm}
\usepackage{tikz}
\usetikzlibrary{matrix}

\newcommand{\bigzero}{\mbox{\normalfont\Large\bfseries 0}}

\newcolumntype{C}[1]{>{\centering\arraybackslash}m{#1}}
\newcolumntype{R}[1]{>{\raggedleft\arraybackslash}m{#1}}
 \usepackage{lipsum}
 
 \usepackage[symbol]{footmisc}

\usepackage[active]{srcltx}
\usepackage{graphicx}
\usepackage{epstopdf}
\usepackage{enumitem}   
\usepackage[T1]{fontenc}
\usepackage{amsmath}

\usepackage{footnote}
\usepackage{footmisc}
\makesavenoteenv{tabular}
\makesavenoteenv{table}

\usepackage{bbm}

\usepackage[
bookmarks=true,         
bookmarksnumbered=true, 
colorlinks=true, pdfstartview=FitV, linkcolor=blue, citecolor=blue,
urlcolor=blue]{hyperref}

\newtheorem {theorem}{Theorem}[section]

\newtheorem{lemma}{Lemma}[section]

\renewcommand\footnotemark{}

\def\qed{{\hfill $\square$ \bigskip}}

\usepackage{dsfont}
\date{\vspace{-5ex}}

\begin{document}

\title{Analytic function approximation by path norm regularized deep networks}

\maketitle

\begin{center}
	\bigskip Aleksandr Beknazaryan \footnote{a.beknazaryan@utwente.nl}
	
    \textit{University of Twente}

	\bigskip
\end{center}

\begin{abstract}We show that neural networks with absolute value activation function and with the path norm, the depth, the width and the network weights having logarithmic dependence on $1/\varepsilon$ can $\varepsilon$-approximate functions that are analytic on certain regions of $\mathbb{C}^d$. 
	
	\vskip.2cm

	\vskip.2cm \noindent {\bf Keywords}:
	\noindent  deep neural networks, analytic functions, path norm regularization, exponential convergence
	
	\vskip.2cm 
\end{abstract}

\section{Introduction}
Deep neural networks have found broad applications in many areas and disciplines, such as computer vision, speech and audio recognition and natural language processing. Two of the main characteristics of a given class of neural networks are its complexity and approximating capability. Once the activation function is selected, a class of networks is determined by specification of the network architecture (namely, its depth and width) and the choice of network weights. Hence, the estimation of the complexity of a given class is done by regularizing (one of) those parameters and the approximation properties of obtained regularized classes of networks are then investigated. 

The capability of shallow networks of depth 1 to approximate continuous functions is shown in the universal approximation theorem (\cite{Sc}) and approximations of integrable functions by networks with fixed width are presented in \cite{L}.  Network architecture constrained approximations of analytic functions are given in \cite{W} where it is shown that  ReLU networks with depth depending logarithmically on $1/\varepsilon$ and width $d+4$ can  $\varepsilon$-approximate analytic functions on the closed subcubes of $(-1,1)^d$. 

Weight regularization of networks is usually done by imposing an $l_p$-related constraint on network weights, $p\geq 0$. The most popular types of such constraints include the $l_0$, $l_1$ and the \textit{path norm} regularizations (see, respectively, \cite{SH}, \cite{TXL} and \cite{N} and references therein). Approximations of $\beta$-smooth functions on $[0,1]^d$ by $l_0$-regularized sparse ReLU networks are given in \cite{SH} and \cite{Y} and exponential rates of approximations of analytic functions by $l_0$-regularized networks are derived in \cite{O}. 

Path norm regularized classes of deep ReLU networks are considered in \cite{N}, where together with other characteristics, the Rademacher complexities of those classes are estimated. The network size independence of those estimates makes the path norm regularization particularly remarkable. As the estimation only uses the Lipschitz continuity (with Lipschitz constant 1), the idempotency and the non-negative homogeneity of the ReLU function, it can be extended to the networks with the absolute value activation function. Network characteristics similar to the path norm are also considered in the works \cite{BK} and \cite{Zh}, where they are called, respectively, a \textit{variation} and a \textit{basis-path norm}, and  statistical features of classes of networks are described in terms of those characteristics. 

The objective of the present paper is the construction of path norm regularized networks  that exponentially fast approximate analytic functions. Our goal is to achieve such convergence rates with activations that are idempotent, non-negative homogeneous and Lipschitz continuous with Lipschitz constant 1, so that the constructed path norm regularized networks  fall within the scope of network classes studied in \cite{N}. It turns out that networks with absolute value activation function may suit this goal better than the networks with ReLU activation function. More precisely, we show that analytic functions can be $\varepsilon$-approximated by networks with absolute value activation function $a(x)$ and with the path norm, the depth, the width and the weights all depending logarithmically on $1/\varepsilon$. Such approximation holds (i) on any subset $(0,1-\delta]^d\subset(0,1)^d$ for analytic functions on $(0,1)^d$ with absolutely convergent power series; (ii) on the whole hypercube $[0,1]^d$ for functions that can be analytically continued to certian subsets of $\mathbb{C}^d$. Note that as the network weights as well as the total number of weights depend logarithmically on $1/\varepsilon,$ then the $l_1$ weight norms of the constructed approximating deep networks are also of logarithmic dependence on $1/\varepsilon$.

\textit{Notation:} For a matrix $W\in\mathbb{R}^{d_1\times d_2}$ we denote by $|W|\in\mathbb{R}^{d_1\times d_2}$ the matrix obtained by taking the absolute values of the entries of $W$: $|W|_{ij}=|W_{ij}|$. For brevity of presentation we will say that the matrix $|W|$ is the \textit{absolute value of the matrix} $W$ (note that in the literature there are also other definitions of the notion of an absolute value of a matrix). The path norm of a network $f$ is denoted by $\|f\|_\times$. For $\textbf{x}=(x_1,...,x_d)\in\mathbb{R}^d$ and $\textbf{k}=(k_1,...,k_d)\in\mathbb{N}_0^d,$ the degree of the monomial $\textbf{x}^\textbf{k}=x_1^{k_1}\cdot ...\cdot x_d^{k_d}$ is defined to be $\|\textbf{k}\|_1=\sum_{i=1}^dk_i.$ To assure that the matrix-vector multiplications are accomplishable, the vectors from $\mathbb{R}^d$, according to the context, may be treated as matices either from $\mathbb{R}^{d\times 1}$ or from $\mathbb{R}^{1\times d}$. 
\section{The class of approximant networks} Neural networks are constituted of the weight matrices, the biases and the nonlinear activation functions acting neuron-wise in the hidden layers. The biases, also called shift vectors, can be omitted by adding a fixed coordinate $1$ to the input vector and correspondingly modifying the weight matrices. As the definition of the path norm of networks does not assume the presence of shift vectors, we will add a coordinate $1$ to the input vector $\textbf{x}$ and will consider classes of neural networks of the form
$$\mathcal{F}_\alpha(L,\textbf{p})=\{f:[0, 1]^p\to\mathbb{R}^{p_{L+1}}\; |\;\; f(\textbf{x})=W_L\circ \alpha\circ W_{L-1}\circ \alpha\circ...\circ \alpha\circ W_0(1, \textbf{x})\},$$
where $W_i\in\mathbb{R}^{ p_{i+1}\times p_i}$ are the weight matrices, $i=0,...,L,$ and $\textbf{p}=(p_0, p_1,...,p_{L+1})$ is the width vector with $p_0=p+1$. The number of hidden layers $L$ determines the depth of networks from $\mathcal{F}_\alpha(L,\textbf{p})$ and in each layer the activation function $\alpha:\mathbb{R}\to\mathbb{R}$ acts element-wise on the input vector. 
For $f\in\mathcal{F}_\alpha(L,\textbf{p})$ given by 
\begin{equation}\label{f}
f(\textbf{x})=W_L\circ \alpha\circ W_{L-1}\circ \alpha\circ...\circ\alpha\circ W_0(1, \textbf{x}),\end{equation}
let 
\begin{equation}\label{norm}
\|f\|_{\times}:=\bigg\|\prod_{i=0}^L|W_i|\bigg\|_1
\end{equation}
be the \textit{path norm} of $f$, where $\|\cdot\|_1$ denotes the $l_1$ norm of the $p_0(=p+1)$ dimensional vector $\prod_{i=0}^L|W_i|$ obtained as a product of absolute values of the weight matrices of $f$. For $B>0$ let
$$\mathcal{F}_\alpha(L,\textbf{p}, B)=\{f\in\mathcal{F}_\alpha(L,\textbf{p}), \|f\|_{\times}\leq B\}$$ be a path norm regularized subclass of $\mathcal{F}_\alpha(L,\textbf{p})$. As the results obtained in \cite{N} indicate, the path norm regularizations are particularly well suited for networks whose activation function  $\alpha$ is
\begin{itemize}
	\item Lipschitz continuous with Lipschitz constant 1;
	\item idempotent, that is, $\alpha(\alpha(x))=\alpha(x)$, $x\in\mathbb{R}$;
	\item non-negative homogeneous, that is, $\alpha(cx)=c\alpha(x),$ for $c\geq 0$, $x\in \mathbb{R}$.
\end{itemize}
We therefore aim to choose an activation $\alpha$ possessing those properties such that analytic functions can be approximated by networks from $\mathcal{F}_\alpha(L,\textbf{p}, B)$ with a small path norm constraint $B$. The most popular activation functions satisfying the above conditions are the ReLU function $\sigma(x)=\max\{0,x\}$ and the absolute value function $a(x)=|x|$. Below we show that with the absolute value activation function the path norms of approximant networks may be significantly smaller than the path norms of the ReLU networks.

The standard technique of neural network function approximation relies on approximating the product function $(x,y)\mapsto xy$ which then allows to approximate monomials and polynomials of any desired degree. In \cite{Y} the approximation of the product $xy=((x+y)^2-x^2-y^2)/2$ is done by approximating the function $x\mapsto x^2$. The latter is based on the observation that for the triangle wave  
\begin{equation}\label{g_s}
g_s(x)=\smash{\underbrace{g\circ g\circ...\circ g}_{s\ \text{times}}}\vphantom{1}(x),
\end{equation} where $g:[0,1]\to[0,1]$ is defined by
\[g(x)= \begin{cases} 2x, & 0\leq x< 1/2, \\
2(1-x), & 1/2\leq x\leq1, \\
\end{cases}
\]
and for any positive integer $m$, 
$$|x^2-f_m(x)|\leq2^{-2m-2},$$ where
\begin{equation}\label{fm}
f_m(x):=x-\sum_{s=1}^{m}\frac{g_s(x)}{2^{2s}}.
\end{equation}
The approximation of $x^2$ by networks with ReLU activation function $\sigma(x)$ then follows from the representation 
\begin{equation}\label{grep}
g(x)=2\sigma(x)-4\sigma(x-1/2).
\end{equation}
 Thus, in this case we will get matrices containing weights 2 and 4 which will make the path norm of approximant networks big. Note that the same approach is also used in \cite{W} for constructing ReLU network approximations of analytic functions. In \cite{SH} the approximation of the product  
\begin{align*}
xy=h\bigg(\frac{x-y+1}{2}\bigg)-h\bigg(\frac{x+y}{2}\bigg)+\frac{x+y}{2}-\frac{1}{4}
\end{align*} 
is done by approximating the function  $h(x):=x(1-x)$, which, in turn, is based on the observation that  for the triangle wave
\begin{align*}
R^k=T^k\circ T^{k-1}\circ...\circ T^1,
\end{align*}  
where $T^k:[0,2^{2-2k}]\to[0,2^{-2k}]$ is defined by
\begin{equation}\label{T}
T^k(x):=\sigma(x/2)-\sigma(x-2^{1-2k}), 
\end{equation}
and for any positive integer $m$,
\begin{align*}
|h(x)-\sum_{k=1}^mR^k(x)|\leq2^{-m}, \quad x\in[0,1].
\end{align*}
Although in the representation \eqref{T} the coefficients (weights) are all in $[-1,1]$, the approximant $\sum_{k=1}^mR^k(x)$ in this case does not have the factors $2^{-2s}$ presented in the approximant $f_m(x)$ in \eqref{fm}, which again will result in big values of path norms. Therefore, to take advantage of the presence of those diminishing weights, we would like to represent the function $g(x)$ in \eqref{grep} by linear combination of activation functions with smaller coefficients. This is possible if instead of $\sigma(x)$ we deploy the absolute value activation function $a(x)$. Indeed, in this case we have that $g(x)$ can be represented on $[0,1]$ as 
\begin{equation}\label{g}
g(x)=1-2a(x-1/2).
\end{equation}
In the next section we use the above representation \eqref{g} to show that analytic functions can be $\varepsilon$-approximated by networks from $\mathcal{F}_{a}(L,\textbf{p}, B)$ with each of $L, \|\textbf{p}\|_\infty$ and $B$ as well as the network weights having logarithmic dependence on $1/\varepsilon$. As all networks will have the same activation function $a(x)=|x|$, in the following the subscript $a$ will be omitted.

\section{Main results}

We first construct a network with activation function $a(x)$, that for the given $\gamma, m\in\mathbb{N}$ simultaneously approximates all $d$-dimensional monomials of degree less than $\gamma$ up to an error $\gamma^24^{-m}$. The depth of this network has order $m\log_2\gamma$ and its width is of order $m\gamma^{d+1}$. Moreover, the entries of the product of the absolute values of matrices of the network have order at most $\gamma^5$ (note the independence of $m$). 

For $\gamma>0$ let $C_{d,\gamma}$ denote the number of $d$-dimensional monomials $\textbf{x}^\textbf{k}$ with degree $\|\textbf{k}\|_1<\gamma$. Then $C_{d,\gamma}<(\gamma+1)^d$ and the following holds: 
\begin{lemma}\label{Mon}
There exists a network \emph{Mon}$_{m,\gamma}^d\in\mathcal{F}(L,\emph{\textbf{p}})$ with $L\leq\lceil \log_2\gamma \rceil(2m+5)+2,$ $p_0=d+1$, $p_{L+1}=C_{d,\gamma}$ and $\|\emph{\textbf{p}}\|_\infty\leq 6\gamma(m+2)C_{d,\gamma}$ such that
$$\bigg\|\emph{Mon}_{m,\gamma}^d(\emph{\textbf{x}})-(\emph{\textbf{x}}^\emph{\textbf{k}})_{\|\emph{\textbf{k}}\|_1<\gamma}\bigg\|_\infty\leq \gamma^24^{-m}, \quad \emph{\textbf{x}}\in[0,1]^d.$$
Moreover, the entries of the $C_{d,\gamma}\times(d+1)$ - dimensional matrix obtained by multiplying the absolute values of matrices presented in $\emph{Mon}_{m,\gamma}^d$ are all bounded by $144(\gamma+1)^5$. 
\end{lemma}
Taking in the above lemma $\gamma, m=\lceil \log_2\frac{1}{\varepsilon}\rceil,$ we get a network from $\mathcal{F}(L,\textbf{p})$ with $L$ and $\|\textbf{p}\|_\infty$ having logarithmic dependence on $1/\varepsilon$, that simultaneously approximates the monomials of degree at most $\gamma$ with error $\varepsilon$ (up to a logarithmic factor). Moreover, the entries of the product of absolute values of matrices of this network will also have logarithmic dependence on $1/\varepsilon$. Below we use this property to construct neural network approximation of analytic and analytically continuable functions with approximation error $\varepsilon$ and with network parameters having logarithmic order.

\begin{theorem}\label{analyitc}
Let $\normalfont f(\textbf{{\textrm{x}}})=\sum_{\textbf{k}\in\mathbb{N}^d_0}a_{\textbf{k}}\textbf{x}^\textbf{k}$ be an analytic function on $(0,1)^d$ with $\normalfont\sum_{\textbf{k}\in\mathbb{N}^d_0}|a_{\textbf{k}}|\leq F$. Then, for any $\varepsilon,\delta\in(0,1)$ there is a constant $C=C(d,F)$ and a network $F_\varepsilon\in\mathcal{F}(L,\emph{\textbf{p}}, B)$ with $L\leq C(\log_2\frac{1}{\delta})(\log^2_2\frac{1}{\varepsilon}),$ $\|\emph{\textbf{p}}\|_\infty\leq \frac{C}{\delta^{d+1}}(\log_2\frac{1}{\varepsilon})^{d+2}$ and $B\leq \frac{C}{\delta^5}\log_2^{5}\frac{1}{\varepsilon},$ such that 
$$\normalfont |F_\varepsilon(\textbf{x})-f(\textbf{x})|\leq \frac{\varepsilon}{\delta^2}, \quad \textit{for all } \; \textbf{x}\in(0,1-\delta]^d.$$

\end{theorem}
Note that an exponential convergence rate of deep ReLU network approximants on subintervals $(0,1-\delta]^d$  is also given in \cite{W}. In our case, however, not only the depth and the width but also the path norm $\|F_\varepsilon\|_\times$ of the constructed network $F_\varepsilon$ have logarithmic dependence on $1/\varepsilon$. Note that in the above theorem, as $\delta$ approaches to $0,$ both $\|{\textbf{p}}\|_\infty$ and $B,$ as well as the approximation error, grow polynomially on $1/\delta.$ In the next theorem we use the properties of Chebyshev series to derive an exponential convergence rate on the whole hypercube $[0,1]^d$.

Recall that the Chebyshev polynomials are defined as $T_0(x)=1,$ $T_1(x)=x$ and
$$T_{n+1}(x)=2xT_n(x)-T_{n-1}(x).$$ Chebyshev polynomials play an important role in the approximation theory, and, in particular, it is known (\cite{T0}, Theorem 3.1) that if $f$ is Lipschitz continuous on $[-1,1]$ then it
has a unique representation as an absolutely and uniformly convergent Chebyshev series
$$f(x)=\sum_{k=0}^{\infty}a_kT_k(x).$$
Moreover, in case $f$ can be analytically continued to an ellipse $E_{\rho}\subset\mathbb{C}$ with foci $-1$ and $1$ and with the sum of semimajor and semiminor axes equal to $\rho>1,$  then the partial sums of the above Chebyshev series converge to $f$ with geometric rate and the coeffients $a_k$ also decay with geometric rate. This result has been first derived by Bernstein in \cite{Ber} and its extension to the multivariate case has been given in \cite{T}. Note that the condition $z\in E_\rho$ implies that $z^2\in N_{1, h^2},$ where $h=(\rho-\rho^{-1})/2$ and for $d,a>0,$ $N_{d, a}\subset\mathbb{C}$ denotes an open ellipse with foci $0$ and $d$ and the leftmost point $-a$. For $F>0,$ $\rho>1$ and $h=(\rho-\rho^{-1})/2$ let $\mathcal{A}^d(\rho, F)$ be the space of functions $f:[0,1]^d\to\mathbb{R}$ that can be analytically continued to the region $\{\textbf{z}\in\mathbb{C}^d: z_1^2+...+z_d^2\in N_{d, h^2}\}$ and are bounded there by $F$. Using the extension of Bernstein's theorem to the multivariate case we get 
\begin{lemma}\label{app}
Let $\rho\geq 2^{\sqrt{d}}$. For $f\in\mathcal{A}^d(\rho, F)$ there is a constant $C=C(d, \rho, F)$ and a polynomial $$\normalfont p(\textbf{{\textrm{x}}})=\sum_{\|\textbf{k}\|_1\leq \gamma}b_{\textbf{k}}\textbf{x}^\textbf{k}, \quad \textbf{x}\in[0,1]^d,$$ 
with 
\begin{equation}\label{b}
\normalfont |b_{\textbf{k}}|\leq C(\gamma+1)^d
\end{equation}
and
$$\normalfont |f(\textbf{x})-p(\textbf{x})|\leq C\rho^{-\gamma/\sqrt{d}},\quad \textit{for all \;} \textbf{x}\in[0,1]^d.$$
\end{lemma} 
Combining Lemma \ref{Mon} and Lemma \ref{app} we get the following
\begin{theorem}\label{appr}
	Let $\varepsilon\in(0,1)$ and let $\rho\geq 2^{\sqrt{d}}$. For $f\in\mathcal{A}^d(\rho, F)$  there is a constant $C=C(d, \rho, F)$ and a network $F_\varepsilon\in\mathcal{F}(L,\emph{\textbf{p}}, B)$ with $L\leq C\log^2_2\frac{1}{\varepsilon},$ $\|\emph{\textbf{p}}\|_\infty\leq C(\log_2\frac{1}{\varepsilon})^{d+2}$ and $B\leq C(\log_2\frac{1}{\varepsilon})^{2d+5}$ such that 
	$$\normalfont |F_\varepsilon(\textbf{x})-f(\textbf{x})|\leq \varepsilon,\quad \textit{for all \;} \textbf{x}\in[0,1]^d.$$
\end{theorem}

We conclude this part by estimating the $l_1$ weight regularization of networks constructed in Theorem \ref{appr}. First, the total number of weights in those networks is bounded by $(L+1)\|\textbf{p}\|_\infty^2=O(\log_2\frac{1}{\varepsilon})^{2d+6}.$ From \eqref{g} it follows that all the weights of network $\normalfont \textrm{Mon}_{m,\gamma}^d$ from Lemma \ref{Mon} are in $[-2,2]$. In Theorem \ref{appr} the network $F_\varepsilon$ is obtained by adding to a network $\normalfont \textrm{Mon}_{m,\gamma}^d,$ with $\gamma=m=O(\log_2\frac{1}{\varepsilon}),$ a layer with coefficients of partial sums of power series of approximated function. Thus, using \eqref{b}, we get that the $l_1$ weight norm of the network $F_\varepsilon$ constructed in Theorem \ref{appr} has order $O(\log_2\frac{1}{\varepsilon})^{4d+6}$.

\section{Proofs}

In the following proofs $I_k$ denotes identity matrix of size $k\times k$ and all the networks have activation $a(x)=|x|$.
The proof of Lemma \ref{Mon} is based on the following 2 lemmas. 

\begin{lemma}\label{m} For any positive integer $m$, there exists a network 
	\emph{Mult}$_m\in\mathcal{F}(2m+3, \emph{\textbf{p}})$, with $p_0=3,$ $p_{L+1}=1$ and $\|\emph{\textbf{p}}\|_\infty=3m+2,$ such that
	
	\begin{equation}\label{mult}
	\normalfont |\textrm{Mult}_m(x, y)-xy|\leq 3\cdot2^{-2m-3}, \quad \textrm{for all} \; x,y\in[0,1],
	\end{equation}
	and the product of absolute values of the matrices presented in \emph{Mult}$_m$ is equal to 
	$$ \bigg(3\sum_{k=1}^{m}\frac{2^{k}-1}{2^{2k}},2-2^{-m}, 2-2^{-m}\bigg).$$
\end{lemma}
\begin{proof}
	For $k\geq 2$ let $R_k$ denote a row of length $k$ with first entry equal to $-1/2$, last entry equal to $1$ and all other entries equal to $0$. Let $A_k$ be a matrix of size $(k+1)\times k$ obtained by adding the $(k+1)$-th row $R_{k}$ to the indentity matrix $I_k$. That is,
	\[
	A_k=\left(
	\begin{array}{@{} c c @{}}
	\begin{matrix}
	\text{\fontsize{6.5mmm}{6.5mm}\selectfont$I_k$}\quad
	\\
	-\frac{1}{2} \quad 0 \quad 0 \quad ... \quad 0 \quad 1
	\end{matrix}
	\end{array}
	\right).
	\]
	
	Let also $B_k$ denote a matrix of size $k\times k$ given by 
	
	\[B_k=
	\begin{pmatrix}
	\text{\fontsize{9mmm}{9mm}\selectfont$\;\;I_{k-1}$}                                 & \begin{matrix} \;\;\;\;0 \\ \;\;\;\;0 \\ \;\;\;\;\vdots\\\;\;\;\;0 \\ \;\;\;\;0 \end{matrix} \\ 
	\begin{matrix}1 & 0 & 0 & ... & 0 &0\end{matrix}&-2 
	\end{pmatrix}.
	\]
	
	It then follows from \eqref{g} that 
	
	\[B_{m+2}\circ a\circ A_{m+1}\circ...\circ B_3\circ a\circ A_2\binom{1}{x}=
	\begin{pmatrix}
	1
	\\
	x 
	\\
	g_1(x)
	\\
	g_2(x)\\
	\cdot\\
	\cdot\\
	\cdot\\
	g_m(x)
	\end{pmatrix},
	\]
	where $g_s(x)$ is the function defined in \eqref{g_s}, $s=1,...,m$.
	Thus, if $S_{m+2}$ is a row of length $m+2$ defined as
	$$S_{m+2}=\bigg(0, 1, -\frac{1}{2^{2\cdot 1}}, -\frac{1}{2^{2\cdot 2}}, ...,  -\frac{1}{2^{2\cdot m}}\bigg),$$
	then 
	$$S_{m+2}\circ a\circ B_{m+2}\circ a\circ A_{m+1}\circ...\circ a\circ B_3\circ a\circ A_2\binom{1}{x}=f_m(x),$$
	where $f_m$ is defined by \eqref{fm}.
    We have that
	$$|S_{m+2}|\cdot|B_{m+2}|\cdot|A_{m+1}|\cdot...\cdot |B_3|\cdot |A_2|=\bigg(\sum_{k=1}^{m}\frac{2^{k+1}-2}{2^{2k}},2-2^{-m}\bigg).$$
	As $xy=\frac{1}{2}\big((x+y)^2-x^2-y^2\big),$ then in the first layer of $\textrm{Mult}_m$ we will obtain a vector 
	\[ \begin{pmatrix}
	1 & 0 & 0
	\\
	0 & 1 & 0 
	\\
	1 & 0 & 0
	\\
	0 & 0 & 1\\
	1 & 0 & 0\\
	0 & 1 & 1 
	\end{pmatrix}\begin{pmatrix}
	1
	\\
	x
	\\
	y
	
	\end{pmatrix}:=C\begin{pmatrix}
	1
	\\
	x
	\\
	y
	
	\end{pmatrix}=\begin{pmatrix}
	1
	\\
	x 
	\\
	1
	\\
	y\\
	1\\
	x+y\\
	\end{pmatrix}
	\] 
	and will then paralelly apply the network from the first part of the proof to each of the pairs $(1, x)$,  $(1, y)$ and  $(1, x+y).$ More precisely, for a given matrix $M$ of size $p\times q$ let $\tilde{M}$ be a matrix of size $3p\times 3q$ defined as

	\[\tilde{M}=
	\begin{pmatrix}
	M & \bigzero & \bigzero 
	\\ \bigzero & M & \bigzero
	\\ \bigzero & \bigzero & M
	\end{pmatrix}.\]
	
	We then have that 
	
	\[ \bigg(-\frac{1}{2} \;\;\ -\frac{1}{2} \;\;\ \frac{1}{2}\bigg)\circ a\circ\tilde{S}_{m+2}\circ a\circ \tilde{B}_{m+2}\circ a\circ \tilde{A}_{m+1}\circ...\circ \tilde{B}_3\circ a\circ \tilde{A}_2\circ a\circ C\begin{pmatrix}
	1
	\\
	x
	\\
	y
	
	\end{pmatrix}=\frac{1}{2}(f_m(x+y)-f_m(x)-f_m(y)),
	\] 
which together with $|f_m(x)-x^2|<2^{-2m-2}$  and the triangle inequality implies \eqref{mult}. It remains to note that
	$$\bigg(\frac{1}{2} \;\;\ \frac{1}{2} \;\;\ \frac{1}{2}\bigg)\cdot|\tilde{S}_{m+2}|\cdot|\tilde{B}_{m+2}|\cdot|\tilde{A}_{m+1}|\cdot...\cdot |\tilde{B}_3|\cdot |\tilde{A}_2|\cdot |C|=\bigg(3\sum_{k=1}^{m}\frac{2^{k}-1}{2^{2k}},2-2^{-m}, 2-2^{-m}\bigg).$$
\end{proof}	

\begin{lemma}\label{Multr}
	For any positive integer $m$, there exists a network 
	\emph{Mult$^r_m\in\mathcal{F}(L, \textbf{p})$}, with $L=(2m+5)\lceil \log_2r \rceil+1,$ $p_0=r+1,$ $p_{L+1}=1$ and $\|\emph{\textbf{p}}\|_\infty\leq 6r(m+2)+1,$ such that
	
	$$|\emph{Mult}^r_m(\emph{\textbf{x}})-\prod_{i=1}^{r}x_i|\leq r^24^{-m} \quad \textit{for all} \; \;  \; \emph{\textbf{x}}=(x_1,...,x_r)\in[0,1]^r,$$
	and for the $(r+1)$-dimensional vector $J_m^r$ obtained by multiplication of absolute values of matrices presented in $\emph{Mult}^r_m$ we have that  $\|J_m^r\|_\infty\leq144r^4$. 
\end{lemma}
\begin{proof}
	First, for a given $k\in\mathbb{N}$, we construct a network $N^k_{m}\in\mathcal{F}(L, \textbf{p})$ with $L=2m+4,$ $p_0=2k+1$ and $p_{L+1}=k+1,$ such that 
	$$N^k_m(x_1, x_2, ..., x_{2k-1}, x_{2k})=(1, \textrm{Mult}_m(x_1,  x_2), ... , \textrm{Mult}_m(x_{2k-1}, x_{2k})).$$
	In the first layer we obtain a vector for which the first coordinate is $1$ followed by triples $(1, x_{2l-1}, x_{2l})$ $l=1,...,k,$ that is, the vector $(1, 1, x_1, x_2, 1, x_3, x_4, ... , 1, x_{2k-1}, x_{2k})$. $N_m^k$ is then obtained by applying parallelly the network $\textrm{Mult}_m$ to each triple $(1, x_{2l-1}, x_{2l})$ while keeping the first coordinate equal to 1. The product of absolute values of the matrices presented in this construction is a matrix of size $(k+1)\times(2k+1)$ having a form 
	
	\[ \begin{pmatrix}
	1 & 0 & 0 & 0 & 0 &  0 & ... & 0 & 0 & 0
	\\
	a_m & b_m & b_m & 0 & 0 & 0 & ...& 0 & 0 & 0 
	\\
	a_m & 0 & 0 & b_m & b_m & 0 & ... & 0 & 0 & 0 
	\\	\cdot & 	\cdot & \cdot & \cdot & \cdot & \cdot & \cdot& \cdot & \cdot & \cdot
	\\a_m & 0 & 0 & 0 & 0  & 0 &... & 0 & b_m & b_m
	\end{pmatrix},
	\] 
	where $a_m=3\sum_{k=1}^{m}\frac{2^{k}-1}{2^{2k}}$ and $b_m=2-2^{-m}$ are the coordinates obtained in the previous lemma.	Let us now construct the network $\textrm{Mult}_m^r$. The first hidden layer of $\textrm{Mult}_m^r$ computes 
	$$(1, x_1, ... ,x_r)\mapsto (1, x_1, ... , x_r, \smash{\underbrace{1, 1, ... ,1}_{2^q-r}}\vphantom{1}),$$  
	$$ $$
	where $q=\lceil \log_2r \rceil$. We then subsequently apply the networks $N_m^{2^q}, N_m^{2^{q-1}},..., N_m^2$ and in the last layer we mutiply the outcome by $(0, 1)$. From Lemma \ref{m} and triangle inequality we have that $\normalfont |\textrm{Mult}_m(x, y)-tz|\leq 3\cdot2^{-2m-3}+|x-t|+|y-z|,$ for $x,y,t,z\in[0,1]$. Hence, by induction on $q$ we get that $|\textrm{Mult}^r_m({\textbf{x}})-\prod_{i=1}^{r}x_i|\leq 3^q2^{-2m-3}\leq 3r^22^{-2m-3}\leq r^24^{-m}$.
	
	 Note that the product of absolute values of matrices in each network $N^k_m$ has the above form, that is, in each row it has at most 3 nonzero values each of which is less than 2. As the matrices given in the first and the last layer of $\textrm{Mult}_m^r$ also satisfy this property, then each entry of the  product of absolute values of all matrices of $\textrm{Mult}_m^r$ will not exceed $12^{q+2}\leq 144r^4$. 
	
\end{proof}
\textit{Proof of Lemma \ref{Mon}.} We have that if $\|\textbf{k}\|_1=0$ then $\textbf{x}^\textbf{k}=1$ and if $\|\textbf{k}\|_1=1$ then $\textbf{k}$ has only one non-zero coordinate, say, $k_j,$ which is equal to $1$ and $\textbf{x}^\textbf{k}=x_j$. Denote $N=C_{d,\gamma}-d-1$ and let $\textbf{k}^1,...,\textbf{k}^N$ be the multi-indices satisfying $1<\|\textbf{k}^i\|_1<\gamma,$ $i=1,...,N$.  For $\textbf{k}=(k_1,...,k_d)$ with $\|\textbf{k}\|_1>1$,  denote by $\textbf{x}_\textbf{k}$  the $(\|\textbf{k}\|_1+1)$-dimesional vector of the form 
$$\textbf{x}_\textbf{k}=(1, \smash{\underbrace{x_1,...,x_1}_{k_1}}\vphantom{1},...,\smash{\underbrace{x_d,...,x_d}_{k_d}}\vphantom{1}).$$
$$ $$
The first layer of $\textrm{Mon}_{m,\gamma}^d$ computes the $\bigg(d+1+\sum_{i=1}^{N}(\|\textbf{k}^i\|_1+1)\bigg)$-dimensional vector
$$(1, \textbf{x}, \textbf{x}_{\textbf{k}^1},...,\textbf{x}_{\textbf{k}^N})^\intercal$$
by multiplying the input vector by matrix $\Gamma$ of size $\bigg(d+1+\sum_{i=1}^{N}(\|\textbf{k}^i\|_1+1)\bigg)\times(r+1)$.
In the following layers we do not change the first $d+1$ coordinates (by multiplying them by $I_{d+1}$) and to each $\textbf{x}_{\textbf{k}^i}$ we apply in parallel the network $\textrm{Mult}^{\|\textbf{k}^i\|_1}_m$. Recall that in Lemma \ref{Multr} $J_m^r$ denotes the $(r+1)$-dimensional vector obtained from the product of absolute values of matrices of  $\textrm{Mult}^r_m$. We then have that the product of absolute values of matrices of $\textrm{Mon}_{m,\gamma}^d$ has the form 

\[
M=\left(
\begin{array}{ccccc}
\text{\fontsize{6.5mmm}{6.5mm}\selectfont$I_k$}          \\
& \text{\fontsize{4.5mmm}{4.5mm}\selectfont$J_m^{\|\textbf{k}^1\|_1}$}            &   & \textbf{\Huge0}\\
&               & \text{\fontsize{4.5mmm}{4.5mm}\selectfont$J_m^{\|\textbf{k}^2\|_1}$}              \\
& \textbf{\Huge0} &   & \ddots            \\
&               &   &   & \text{\fontsize{4.5mmm}{4.5mm}\selectfont$J_m^{\|\textbf{k}^N\|_1}$} 
\end{array}
\right)\cdot\Gamma.
\]

As the matrix $\Gamma$ only contains entries $0$ and $1$ then applying Lemma \ref{Multr} we get that the entries of $M$ are bounded by 
$$\max\limits_{1\leq i\leq N}\bigg|\bigg|J_m^{\|\textbf{k}^i\|_1}\bigg|\bigg|_1\leq 144(\gamma+1)^5.$$
\qed

\textit{Proof of Theorem \ref{analyitc}} Let $\gamma=\lceil \frac{1}{\delta}\ln \frac{1}{\varepsilon}\rceil.$ Then, for $\textbf{x}\in(0, 1-\delta]^d$ we have that
$$\bigg|f(\textbf{{\textrm{x}}})-\sum_{\|\textbf{k}\|_1\leq\gamma}a_{\textbf{k}}\textbf{x}^\textbf{k}\bigg|=\bigg|\sum_{\|\textbf{k}\|_1>\gamma}a_{\textbf{k}}\textbf{x}^\textbf{k}\bigg|\leq(1-\delta)^\gamma F\leq\varepsilon F.$$
In order to approximate the partial sum $\sum_{\|\textbf{k}\|_1\leq\gamma}a_{\textbf{k}}\textbf{x}^\textbf{k},$ we add one last layer with the coefficients of that partial sum to the network $\textrm{Mon}_{m,\gamma+1}^d$ obtained in Lemma \ref{Mon} with $m=\log_2\lceil \frac{1}{\varepsilon}\rceil$. For the obtained network $F_\varepsilon$ we have that
$$\|F_\varepsilon\|_\times\leq 144(d+1)F(\gamma+2)^5.$$\qed

Let us now present the result from \cite{T} that will be used to derive Lemma \ref{app}.
First, if $f\in\mathcal{A}^d(\rho, F)$, then (\cite{M}, Theorem 4.1) $f$ has a unique representation as an absolutely and uniformly convergent multivariate Chebyshev series
$$f(\textbf{x})=\sum_{k_1=0}^{\infty}...\sum_{k_d=0}^{\infty}a_{k_1, ..., k_d}T_{k_1}(x_1)...T_{k_d}(x_d), \quad \textbf{x}\in[0,1]^d.$$
Note that for $\mathbf{k}:=(k_1,...,k_d)$, the degree of a $d$-dimesional polynomial $T_{k_1}(x_1)...T_{k_d}(x_d)$ is $\|\textbf{k}\|_1=k_1+...+k_d$.
Then, for any non-negative integers $n_1,...,n_d,$ the partial sum 
\begin{equation}\label{partial}
p(\textbf{x})=\sum_{k_1=0}^{n_1}...\sum_{k_d=0}^{n_d}a_{\textbf{k}}T_{k_1}(x_1)...T_{k_d}(x_d)
\end{equation}
is a polynomial truncation of the multivariate Chebyshev series of $f$ of degree $d(p)=n_1+...+n_d$. It is shown in \cite{T} that
\begin{theorem}
	For $f\in\mathcal{A}^d(\rho, F)$ there is a constant $C=C(d, \rho, F)$ such that the multivariate Chebyshev coefficients of $f$ satisfy 
	\begin{equation}\label{coef}
\normalfont	|a_{\textbf{k}}|\leq C\rho^{-\|\textbf{k}\|_2}
	\end{equation}
	and for the polynomial truncations $p$ of the multivariate Chebyshev series of $f$ we have that
	$$\normalfont\inf_{d(p)\leq\gamma}\|f(\textbf{x})-p(\textbf{x})\|_{[0,1]^d}\leq C\rho^{-\gamma/\sqrt{d}}.$$
\end{theorem}
\textit{Proof of Lemma \ref{app}}
Note that from the recursive definition of the Chebyshev polynomials it follows that for any $k\geq 0$ the coefficients of the Chebyshev polynomial $T_k(x)$ are all bounded by $2^k$. Let now $p$ be a polynomial given by \eqref{partial} with degree $d(p)\leq\gamma$. As the number of summands in the right-hand side of \eqref{partial} is bounded by $(\gamma+1)^d,$ then, using \eqref{coef}, we get that $p$ can be rewritten as
$$p(\textbf{x})=\sum_{\|\textbf{k}\|_1\leq \gamma}b_{\textbf{k}}\textbf{x}^\textbf{k},$$ 
with 
$$|b_{\textbf{k}}|\leq C(\gamma+1)^d2^{\|\textbf{k}\|_1}\rho^{-\|\textbf{k}\|_2}\leq C(\gamma+1)^d2^{\sqrt{d}\|\textbf{k}\|_2}\rho^{-\|\textbf{k}\|_2}\leq C(\gamma+1)^d,$$
where the last inequality follows from the condition $\rho\geq 2^{\sqrt{d}}$.\qed

\textit{Proof of Theorem \ref{appr}} The proof follows from Lemma \ref{Mon} and Lemma \ref{app} by taking $\gamma=m=\lceil\log_2\frac{1}{\varepsilon}\rceil$ and adding to the network $\textrm{Mon}_{m,\gamma+1}^d$ the last layer with the coefficients of the polynomial $p(\textbf{x})$ from Lemma  \ref{app}. For the obtained network $F_\varepsilon$ we have that 
$$\|F_\varepsilon\|_\times\leq144C(d+1)C_{d,\gamma+1}(\gamma+2)^d(\gamma+2)^{5}\leq144C(d+1)(\gamma+2)^{2d+5},$$
where $C$ is the constant from Lemma \ref{app}.

\section*{Acknowledgement} The author would like to thank Johannes Schmidt-Hieber for support and valuable suggestions. The work has been supported by the NWO Vidi grant: ``\textit{Statistical foundation for multilayer neural networks}''.

\end{document}